\documentclass[10pt]{article} 
\usepackage[preprint]{tmlr}

\usepackage{hyperref}
\usepackage{url}

\usepackage{graphicx} 
\usepackage{amssymb}
\usepackage{stmaryrd}
\usepackage{amsmath}
\usepackage{amsthm}

\newtheorem{theorem}{Theorem}
\newtheorem{lemma}{Lemma}

\title{Tracking the Best Strategy in an Extensive-Form Game}

\author{\name Stephen Pasteris \email spasteris@turing.ac.uk \\
      \addr The Alan Turing Institute\\
      London, United Kingdom
      \AND
      \name Rahul Savani \email rahul.savani@liverpool.ac.uk \\
      \addr The Alan Turing Institute;\\
  The University of Liverpool \\
  Liverpool, United Kingdom
      \AND
      \name Theodore Turocy \email t.turocy@uea.ac.uk\\
      \addr The Alan Turing Institute;\\
  The University of East Anglia \\
  Norwich, United Kingdom}

\def\month{MM}  
\def\year{YYYY} 
\def\openreview{\url{https://openreview.net/forum?id=XXXX}} 

\begin{document}

\maketitle

\begin{abstract}
We consider the extensive-form bandit problem where on each trial the learner plays an extensive-form game against an oblivious adversary. We focus on the notion of switching regret, which measures the expected performance of the learner against that of any switching sequence of mixed strategies in retrospect. Our algorithm takes a parameter $\rho>0$ and achieves a switching regret of $\tilde{\mathcal{O}}((1/\rho+\rho K)\sqrt{H A T})$ where $K$ is the number of switches in the comparator sequence, $H$ is the maximum number of the learner's information sets that can be traversed during a play of the game and $A$ is the number of actions that the learner can possibly take. Our algorithm is extremely efficient, taking a per trial time of only $\mathcal{O}(H B)$ where $B$ is the maximum number of actions available to the learner at any of its information sets.
\end{abstract}

\newcommand{\nc}[1]{\newcommand{#1}}
\nc{\losf}{\lambda}
\nc{\tns}{\mathcal{L}}
\nc{\ver}{\mathcal{V}}
\nc{\lns}{\mathcal{N}}
\nc{\nns}{\mathcal{A}}
\nc{\ch}[1]{\mathcal{C}(#1)}
\nc{\anod}{v}
\nc{\rot}{r}
\nc{\ac}{a}
\nc{\pnt}[1]{p(#1)}
\nc{\st}{\sigma}
\nc{\rf}{\mu}
\nc{\sts}{\mathcal{S}}
\nc{\rfs}{\mathcal{M}}
\nc{\loss}[2]{\Lambda(#1,#2)}
\nc{\lrh}{\rho}
\nc{\nsw}{K}
\nc{\depth}{H}
\nc{\tnma}{A}
\nc{\nma}{B}
\nc{\be}{\begin{equation*}}
\nc{\ee}{\end{equation*}}
\nc{\alg}{\textsc{TrackEFG}}
\nc{\pol}{\pi}
\nc{\pols}{\mathcal{P}}
\nc{\indi}[1]{\llbracket #1 \rrbracket}
\nc{\rft}[1]{\mu_{#1}}
\nc{\lst}[1]{\sigma_{#1}}
\nc{\lsts}{\sigma}
\nc{\lost}[1]{\ell_{#1}}
\nc{\reg}[1]{R(#1)}
\nc{\nsws}[1]{K(#1)}
\nc{\cpols}{\vartheta}
\nc{\cpolt}[1]{\vartheta_{#1}}
\nc{\expt}[1]{\mathbb{E}\left[ #1\right]}
\nc{\mgp}[2]{m_{#1}(#2)}
\nc{\mgt}[2]{g_{#1}(#2)}

\section{Introduction}
We consider a (perfect recall) extensive form game. We call any mixed strategy of the learner a \emph{policy} and any deterministic realisation of how the learner's opponents (including any ``chance'' player) play an \emph{environment}. The extensive-form bandit problem consists of $T$ trials where on each trial we have an unknown environment $\rft{t}$ fixed a-priori. On each trial $t$ the learner stochastically plays the game against the enviroment $\rft{t}$\,, observing its information sets (a.k.a. \emph{infosets}) that are traversed. We denote the loss (i.e. the negated payoff) incurred by the learner on trial $t$ by $\lost{t}$.

The extensive-form bandit problem has been well studied by works such as \cite{ Kozuno2021ModelFreeLF, Bai2022NearOptimalLO, Fiegel2022AdaptingTG, Farina2021BanditLO, Maiti2025EfficientNA} which give bounds on the \emph{static regret}, which is the cumulative loss incurred by the learner minus that which it would have obtained by always playing according to the best fixed policy in retrospect. In this paper we focus on the more general notion of \emph{switching regret}, which is the cumulative loss incurred by the learner minus that which would have been expected by playing an arbitrary sequence of policies in retrospect. Formally, given a policy $\pol$ and an environment $\rf'$, we define $\loss{\pol}{\rf'}$ to be the expected loss incurred by the learner if it were to play the policy $\pol$ against the environment $\rf'$. Given a sequence of policies $\cpols\in\pols^T$ we define the \emph{(switching) regret} with respect to $\cpols$ as:
\be
\reg{\cpols}:=\sum_{t\in[T]}\lost{t}-\sum_{t\in[T]}\loss{\cpolt{t}}{\rft{t}}\,.
\ee
Given a sequence of policies $\cpols\in\pols^T$ we define:
\be
\nsws{\cpols}:=1+|\{t\in[T-1]\,|\,\cpolt{t+1}\neq\cpolt{t}\}|
\ee
which is the number of times the policy changes throughout the sequence. In this paper we give an algorithm \alg\ for the learner that takes a parameter $\lrh>0$ and achieves, for any sequence of policies $\cpols\in\pols^T$, an expected regret of:
\be
\mathbb{E}[\reg{\cpols}]\in\tilde{\mathcal{O}}\left(\left(\frac{1}{\lrh}+\lrh\nsws{\cpols}\right)\sqrt{\depth \tnma T}\right)
\ee
where $\depth$ is the maximum number of the learner's infosets that can be traversed during a play of the game and $\tnma$ is the number of actions that the learner can possibly take (i.e. the sum of the number of actions avaliable at each of the learner's infosets). A salient feature of \alg\ is its extreme computational efficiency, taking a per trial time of only $\mathcal{O}(\depth \nma)$ where $\nma$ is the maximum number of actions available at one of the learner's infosets.

From an algorithmic perspective \alg\ is simple - it essentially follows  \textsc{BalancedOMD} \cite{Bai2022NearOptimalLO} (we note though that we have replaced the use of the \emph{balanced exploration policy} for performance improvements) but at the end of each trial  we apply the \textsc{FixedShare} \cite{Herbster1995TrackingTB} update to the probability distribution over actions at each infoset visited on that trial. The analysis, however, is not a simple combination of (as far as we are aware) known analyses of \textsc{BalancedOMD} and \textsc{FixedShare}.

\subsection{Related Work}
The notion of switching regret was introduced in \cite{Herbster1995TrackingTB}, which, generalising the \textsc{Hedge} algorithm of \cite{Freund1997ADG}, gave the algorithm \textsc{FixedShare} for the problem of prediction with expert advice.
\textsc{FixedShare} was modified by \cite{Auer2002TheNM} to create the \textsc{Exp3.S} algorithm for non-stationary learning in the adversarial bandit problem. In the case in which the learner has a single infoset, our problem reduces to the adversarial bandit problem and our algorithm \alg\ reduces to \textsc{Exp3.S}.

The work \cite{Maiti2025EfficientNA} reduces the extensive-form bandit problem to adversarial online shortest path on a directed acyclic graph under semi-bandit feedback. For this problem the works \cite{Benk2007TheOS, Vural2020ShortestPL} give switching regret guarantees. We suspect that a refined analysis of \cite{Benk2007TheOS}, specific to extensive-form games, would lead to a regret of:
\be
\mathbb{E}[\reg{\cpols}]\in\tilde{\mathcal{O}}\left(\left(\frac{1}{\lrh}+\lrh\nsws{\cpols}\right)\sqrt{\ln(S)X T}\right)
\ee
where $S$ is the (massive) number of the learner's reduced strategies in the reduced strategic form of the game and $X$ is the maximum number of leaves that are reachable under any environment. Whilst this bound would be, in general, incomparable to ours, the algorithm has an extremely high per-trial time and space complexity of $\mathcal{O}(\tnma T)$. The work \cite{Vural2020ShortestPL} gives an algorithm with a regret of:
\be
\mathbb{E}[\reg{\cpols}]\in\tilde{\mathcal{O}}\left(\left(\frac{1}{\lrh}+\lrh\nsws{\cpols}\right)\sqrt{Y\tnma T}\right)
\ee
where $Y$ is the maximum number of the learner's infosets that are reachable under any of the learner's pure strategies. Not only is this regret bound significantly worse than ours but the algorithm has an extremely high per-trial time complexity of $\mathcal{O}(\tnma^4)$.

One may wonder whether the corralling machinery of \cite{Luo2022CorrallingAL} could be applied to \textsc{BalancedOMD} \cite{Bai2022NearOptimalLO} which efficiently obtains static regret bounds for the extensive-form bandit problem. However, there are issues in doing this. The first issue is that, to work with \cite{Luo2022CorrallingAL}, \textsc{Exp3} had to be modified which means that \textsc{BalancedOMD}, which generalises \textsc{Exp3}, would also need to be modified. The second issue is that, as far as we are aware, on any trial, the bias term would be dependent on the entire environment at that trial, which is unknown at the end of the trial. We believe that the bias term could be overestimated but this would lead to a dramatically worse regret bound than us and take a time of $\mathcal{O}(\tnma)$ to compute. We believe the same issues to hold when trying to corral the algorithms of \cite{Kozuno2021ModelFreeLF, Fiegel2022AdaptingTG}. With regard to the algorithms of \cite{Maiti2025EfficientNA, Farina2021BanditLO}, even if these algorithms could be corralled, they would suffer from a dramatically higher regret and time complexity than \alg.

Whilst the works \cite{Kozuno2021ModelFreeLF, Bai2022NearOptimalLO} are based on \emph{dilated mirror descent} \cite{Hoda2010SmoothingTF, Farina2021BetterRF}, there is also a line of work \cite{Lanctot2009MonteCS, Farina2020StochasticRM, Bai2022NearOptimalLO} on \emph{monte-carlo counterfactual regret minimisation}. However, whilst achieving the goal of finding a Nash equilibrium (for a two player zero-sum game), these works fail to give a bound on the true static regret, as the game is played with a policy that is different from the policy that the ``regret'' is measured against.

The work \cite{Noarov2023HighDimensionalPF} considered the notion of subsequence regret, in which switching regret can be seen as a special case. However, their algorithm is full information (in that it must see the entirety of the environment at the end of each trial) and their external regret bound is at least $\tilde{\mathcal{O}}(\tnma\sqrt{T})$ which is high - especially for a full information algorithm. The per-trial time complexity of their algorithm is also at least $\mathcal{O}(T^2)$.

Related to the extensive-form bandit problem is the problem of \emph{reinforcement learning} (where transitions are stochastic rather than adversarial). The work \cite{Wei2021NonstationaryRL} studied reinforcement learning where the transition probabilities change gradually over time.

The development and analysis of \alg\ was inspired by the works \cite{Pasteris2023NearestNW, Herbster1995TrackingTB, Bai2022NearOptimalLO, Pasteris2026DifferentialPI}.

\nc{\arbs}{\mathcal{B}}

\subsection{Definitions}
Let $\mathbb{N}$ be the set of natural numbers excluding $0$. For all $i\in\mathbb{N}$ let $[i]:=\{j\in\mathbb{N}\,|\,j\leq i\}$. Given a predicate $P$ let $\indi{P}$ be equal to $1$ if $P$ is true and equal to $0$ otherwise.

\section{The Game}

For this paper we need not fully define an imperfect information extensive-form game as we only need to focus on the leaner's infosets and actions and not those of its opponents. In this one-sided view of the game we have a rooted tree whose nodes are either an infoset of the learner, an action of the learner, or a terminal node (a.k.a. \emph{leaf}). The children of each infoset are the set of actions that are avaliable at that infoset and the children of each action are the possible leaves or infosets (of the learner) that can be encountered after that action is taken.

Formally, we have a rooted tree and a function $\losf:\tns\rightarrow[0,1]$ where $\tns$ is the set of leaves of the tree. Let $\ver$ be the set of nodes of the tree and let $\rot$ be the root of the tree. Given a node $\anod\in\ver$, let $\ch{\anod}$ be the set of its children and, given $\anod\neq\rot$, let $\pnt{\anod}$ be its parent. The set $\ver\setminus\tns$, of internal nodes of the tree, is partitioned into two sets $\lns$ (the set of the learner's \emph{infosets}) and $\nns$ (the set of the learner's \emph{actions}) satisfying the following rules:
\begin{itemize}
    \item $\rot\in\lns$.
    \item For all $\anod\in\lns$ we have $\ch{\anod}\subseteq\nns$.
    \item For all $\ac\in\nns$ we have $\ch{\ac}\subseteq\lns\cup\tns$.
\end{itemize}
Let $\depth$ be the maximum number of nodes in $\nns$ in any root-to-leaf path. Let:
\be
\tnma:=|\nns|~~~~~,~~~~~\nma:=\max_{\anod\in\lns}|\ch{\anod}|\,.
\ee

 An \emph{environment} is defined as a function $\rf:\nns\rightarrow\ver$ with $\rf(\ac)\in\ch{\ac}$ for all $\ac\in\nns$. Let $\rfs$ be the set of all possible environments. We note that, in this paper, we will, without loss of generality, consider only deterministic environments, as to handle stochastic environments one simply draws an environment from a probability distribution (so our regret bound will hold for stochastic environments as well).

Given some $\rf\in\rfs$, when the learner plays the game against $\rf$ a root-to-leaf path is traversed as follows. We start at node $\rot$ and:
\begin{itemize}
\item When we are at a node $\anod\in\lns$, the learner observes $\anod$ and must choose an action $\ac\in\ch{\anod}$ to take. We move next to node $\ac$.
\item When at a node $\ac\in\nns$\,, we move next to node $\rf(\ac)$.
\item When at a node $\anod\in\tns$ we terminate and the learner observes and incurs \emph{loss} $\losf(\anod)$
\end{itemize}

A \emph{policy} is defined as a function $\pol:\nns\rightarrow[0,1]$ such that for all $\anod\in\lns$ we have:
\be
\sum_{\ac\in\ch{\anod}}\pol(\ac)=1\,.
\ee
Let $\pols$ be the set of all policies. Given $\pol\in\pols$, when playing the game against an environment, we say that the learner plays according to policy $\pol$ if and only if, when at a node $\anod\in\lns$ it draws an action $\ac\in\ch{\anod}$ independently with probability $\pol(\ac)$ and takes action $\ac$. Given $\pol\in\pols$ and $\rf\in\rfs$ we define $\loss{\pol}{\rf}$ to be the expected loss of the learner if it were to play the game against $\rf$ according to policy $\pol$.

\section{The Problem and Result}

Our problem consists of $T$ trials where on each trial we have an unknown (to the learner) environment $\rft{t}$ fixed a-priori. The learner knows the tree a-priori but not necessarily the function $\losf$. For each trial $t\in[T]$ in turn the learner plays the game against $\rft{t}$. Let $\lost{t}$ be the loss incurred by the learner on trial $t$. The aim of the learner is to minimise the cumulative loss incurred.

In this paper we give an algorithm \alg\ for the learner that takes a parameter $\lrh>2/\sqrt{T}$. In order to present our main theorem we make the following definitions.

Given a sequence of policies $\cpols\in\pols^T$ we define the \emph{regret} with respect to $\cpols$ as:
\be
\reg{\cpols}:=\sum_{t\in[T]}\lost{t}-\sum_{t\in[T]}\loss{\cpolt{t}}{\rft{t}}
\ee
which is the difference between the cumulative loss of the learner and that which would have been expected if, on each trial $t\in[T]$, it had played according to policy $\cpolt{t}$. 

Given a sequence of policies $\cpols\in\pols^T$ we define:
\be
\nsws{\cpols}:=1+|\{t\in[T-1]\,|\,\cpolt{t+1}\neq\cpolt{t}\}|
\ee
which is the number of times the policy changes throughout the sequence.

We now give the following theorem about \alg.

\begin{theorem}\label{mainth}
For any sequence of policies $\cpols\in\pols^T$, \alg\ achieves:
\be
\mathbb{E}[\reg{\cpols}]\in\mathcal{O}\left(\left(\frac{1}{\lrh}+\lrh\nsws{\cpols}\right)\sqrt{\depth \tnma T\ln(\lrh^2T)}\right)
\ee
where the expectation is over the randomisation in \alg. \alg\ has a per trial time complexity of $\mathcal{O}(\depth \nma)$ and has an initialisation time and space complexity of $\mathcal{O}(\tnma)$.
\end{theorem}

\begin{proof}
See Section \ref{anasec}.
\end{proof}

\nc{\msn}{n}
\nc{\ime}{\beta}
\nc{\acs}{\nns}
\nc{\polt}[1]{\pi_{#1}}
\nc{\lr}{\eta}
\nc{\shp}{\phi}
\nc{\spolt}[1]{\tilde{\pi}_{#1}}
\nc{\plt}[1]{d_{#1}}
\nc{\vet}[2]{x_{#1,#2}}
\nc{\act}[2]{b_{#1,#2}}
\nc{\pst}[2]{\psi_{#1,#2}}
\nc{\mst}[1]{\polt{#1}}
\nc{\msh}[1]{\pi'_{#1}}
\nc{\mss}{\pols}
\nc{\lut}[1]{\tau(#1)}

\section{The Algorithm}

To initialise the algorithm we first construct a function $\msn:\lns\cup\acs\rightarrow\mathbb{N}$ recursively (up the tree) as follows:
\begin{itemize}
\item For all $\ac\in\acs$ we have:
\be
\msn(\ac):=1+\sum_{\anod\in\ch{\ac}\cap\lns}\msn(\anod)\,.
\ee
\item For all $\anod\in\lns$ we have:
\be
\msn(\anod):=\sum_{\ac\in\ch{\anod}}\msn(\ac)\,.
\ee
\end{itemize}
We note that for all $\anod\in\lns\cup\acs$ we have that $\msn(\anod)$ is the number of nodes in $\acs$ that are descendants of $\anod$.

We then construct a function $\ime:\lns\cup\acs\rightarrow\mathbb{R}^+$ recursively (down the tree) as follows:
\begin{itemize}
\item $\ime(\rot):=1$
\item For all $\ac\in\acs$ we have:
\be
\ime(\ac):=\msn(\ac)\ime(\pnt{\ac})\,.
\ee
\item For all $\anod\in\lns\setminus\{\rot\}$ we have:
\be
\ime(\anod):=\frac{\ime(\pnt{\anod})}{\msn(\anod)}\,.
\ee
\end{itemize}

We then define:
\be
\lr:=\frac{1}{\lrh}\sqrt{\frac{2\ln(\lrh^2T)}{\depth \msn(\rot)T}}~~~~~,~~~~~\shp := \frac{1}{\lrh^2 T}\,.
\ee

\alg\ maintains a dynamic (in that it changes from trial to trial) policy. Let $\polt{t}\in\pols$ be the value of this policy at the start of trial $t$. $\polt{1}$ is defined so that for all $\anod\in\lns$ and all $\ac\in\ch{\anod}$ we have:
\be
\polt{1}(\ac):=\frac{1}{|\ch{\anod}|}\,.
\ee
On a trial $t\in[T]$ the learner plays the game (against $\rft{t}$) according to policy $\polt{t}$. Let $\plt{t}$ be the number of nodes in $\lns$ that were encountered during the play of the game on trial $t$. For all $i\in[\plt{t}]$ let $\vet{t}{i}\in\lns$ be the $i^{\text{th}}$ node in $\lns$ that was encountered during the play of the game on trial $t$ and let $\act{t}{i}\in\ch{\vet{t}{i}}$ be the action that was taken by the learner at $\vet{t}{i}$ on trial $t$. Recall that $\lost{t}$ is the loss incurred by the learner on trial $t$.

We now describe how $\polt{t}$ is updated to $\polt{t+1}$ at the end of trial $t$. First define:
\be
\pst{t}{\plt{t}}:=\exp\left(\frac{-\lr\lost{t}}{\prod_{i\in[\plt{t}]}\polt{t}(\act{t}{i})}\right)\,.
\ee
For all $i\in[\plt{t}]$, once $\pst{t}{i}$ has been defined, we define:
\be
\pst{t}{i-1}:=\left(1-\left(1-\pst{t}{i}^{\ime(\act{t}{i})}\right)\mst{t}(\act{t}{i})\right)^{1/\ime(\act{t}{i})}\,.
\ee
For all $i\in[\plt{t}]$ and all $\ac\in\ch{\vet{t}{i}}$ we then define:
\be
\polt{t+1}(\ac):=\frac{\shp}{|\ch{\vet{t}{i}}|}+\frac{(1-\shp)\polt{t}(\ac)}{\pst{t}{i-1}^{\ime(\act{t}{i})}}\left(\indi{\ac=\act{t}{i}}\pst{t}{i}^{\ime(\act{t}{i})}+\indi{\ac\neq\act{t}{i}}\right)
\ee
and for all $\anod\in\lns\setminus\{\vet{t}{i}\,|\,i\in[\plt{t}]\}$ and $\ac\in\ch{\anod}$ we maintain:
\be
\polt{t+1}(\ac):=\mst{t}(\ac)\,.
\ee
This completes the description of \alg.

Of course, we must prove that, for all $t\in[T]$, $\polt{t}$ is indeed a policy. This is confirmed by the following theorem.

\begin{theorem}
For all $t\in[T+1]$ we have $\polt{t}\in\pols$.
\end{theorem}

\begin{proof}
We prove by induction over $t$. Since it is clear that $\polt{1}\in\pols$ all we need to do is prove that, for any $t\in[T]$ with $\polt{t}\in\pols$, we have $\polt{t+1}\in\pols$ also. So assume we have $t\in[T]$ with $\polt{t}\in\pols$.

Using the fact that $\polt{t}(\ac)\in[0,1]$ for all $\ac\in\acs$, by a simple backward induction over $j$ we see that for all $j\in[\plt{t}]$ we have $\pst{t}{j}\geq0$ and hence that $\pst{t}{j}^{\ime(\act{t}{j})}\geq0$. Hence, for all $\ac\in\acs$, we have, since $\polt{t}(\ac)\geq 0$, that $\polt{t+1}(\ac)\geq0$.

Now take any $i\in[\plt{t}]$. Since $\polt{t}\in\pols$ we have:
\be
\sum_{\ac\in\ch{\vet{t}{i}}}\polt{t}(\ac)=1
\ee
so that:
\begin{align*}
\sum_{\ac\in\ch{\vet{t}{i}}}\polt{t}(\ac)\left(\indi{\ac=\act{t}{i}}\pst{t}{i}^{\ime(\act{t}{i})}+\indi{\ac\neq\act{t}{i}}\right)&=\polt{t}(\act{t}{i})\pst{t}{i}^{\ime(\act{t}{i})}+\sum_{\ac\in\ch{\vet{t}{i}}}\polt{t}(\ac)\indi{\ac\neq\act{t}{i}}\\
&=\polt{t}(\act{t}{i})\pst{t}{i}^{\ime(\act{t}{i})}+\sum_{\ac\in\ch{\vet{t}{i}}}\polt{t}(\ac)-\polt{t}(\act{t}{i})\\
&=\polt{t}(\act{t}{i})\pst{t}{i}^{\ime(\act{t}{i})}+1-\polt{t}(\act{t}{i})\\
&=1-\left(1-\pst{t}{i}^{\ime(\act{t}{i})}\right)\mst{t}(\act{t}{i})\\
&=\pst{t}{i-1}^{\ime(\act{t}{i})}
\end{align*}
and hence:
\begin{align*}
\sum_{\ac\in\ch{\vet{t}{i}}}\polt{t+1}(\ac)&=\sum_{\ac\in\ch{\vet{t}{i}}}\frac{\shp}{|\ch{\vet{t}{i}}|}+\frac{1-\shp}{\pst{t}{i-1}^{\ime(\act{t}{i})}}\sum_{\ac\in\ch{\vet{t}{i}}}\polt{t}(\ac)\left(\indi{\ac=\act{t}{i}}\pst{t}{i}^{\ime(\act{t}{i})}+\indi{\ac\neq\act{t}{i}}\right)\\
&=\shp+(1-\shp)\\
&=1\,.
\end{align*}
For all $\anod\in\lns\setminus\{\vet{t}{i}\,|\,i\in[\plt{t}]\}$ we have, since $\polt{t}\in\pols$, that:
\begin{align*}
\sum_{\ac\in\ch{\anod}}\polt{t+1}(\ac)&=\sum_{\ac\in\ch{\anod}}\polt{t}(\ac)\\
&=1\,.
\end{align*}
We have now shown that for all $\anod\in\lns$ we have:
\be
\sum_{\ac\in\ch{\anod}}\polt{t+1}(\ac)=1
\ee
so since, by above, we have that $\polt{t+1}(\ac)>0$ for all $\ac\in\acs$, we have that $\polt{t+1}\in\pols$ which completes the inductive proof.
\end{proof}

\section{Analysis}\label{anasec}

\nc{\rac}[1]{\mathcal{A}^\dag_{#1}}
\nc{\wgt}[2]{\omega_{#1}(#2)}
\nc{\acu}{\mathcal{U}}

We will now prove Theorem \ref{mainth}.

For all $t\in[T]$ and all $i\in[\depth]\setminus[\plt{t}]$ we define $\vet{t}{i}$ and $\act{t}{i}$ to be equal to some mathematical object not contained in $\ver$.  

A \emph{(pure) strategy} is defined as a function $\st:\lns\rightarrow\acs$ with $\st(\anod)\in\ch{\anod}$ for all $\anod\in\lns$. Let $\sts$ be the set of all strategies. Given some $\st\in\sts$, when playing the game against an environment, we say that the learner plays according to strategy $\st$ if and only if, when at a node $\anod\in\lns$, it takes action $\st(\anod)$. Given some $\st\in\sts$ and $\rf\in\rfs$, let $\loss{\st}{\rf}$ be the loss incurred by the learner if it were to play the game against $\rf$ according to strategy $\st$.

Given a sequence of strategies $\lsts\in\sts^T$ we define:
\be
\nsws{\lsts}=1+|\{t\in[T-1]\,|\,\lst{t+1}\neq\lst{t}\}\,.
\ee

\begin{lemma}\label{poltostlem}
For any sequence of policies $\cpols\in\pols^T$ there exists a sequence of strategies $\lsts\in\sts^T$ in which:
\be
\nsws{\lsts}\leq\nsws{\cpols}
\ee
and:
\be
\sum_{t\in[T]}\loss{\lst{t}}{\rft{t}}\leq\sum_{t\in[T]}\loss{\cpolt{t}}{\rft{t}}\,.
\ee
\end{lemma}

\nc{\ssps}{\vartheta^\dag}
\nc{\ssp}[1]{\vartheta^\dag_{#1}}
\nc{\swp}[1]{\tau_{#1}}
\nc{\swps}{\tau}
\nc{\sss}[1]{\sigma^\dag_{#1}}
\nc{\astr}{\sigma'}
\nc{\pra}[2]{z_{#1}(#2)}
\nc{\arf}{\mu'}
\nc{\sht}[2]{\chi_{#1}(#2)}

\begin{proof}
Note first that there exists a sequence of policies $\ssps\in\pols^{\nsws{\cpols}}$ and a sequence $\swps\in[T+1]^{\nsws{\cpols}+1}$ such that:
\begin{itemize}
\item $\swp{1}=1$ and $\swp{\nsws{\cpols}+1}=T+1$.
\item For all $i\in[\nsws{\cpols}]$ we have $\swp{i+1}>\swp{i}$.
\item For all $i\in[\nsws{\cpols}]$ and all $t\in[\swp{i},\swp{i+1}-1]\cap\mathbb{N}$ we have $\cpolt{t}=\ssp{i}$.
\end{itemize}
For all $i\in[\nsws{\cpols}]$ define:
\be
\sss{i}:=\operatorname{argmin}_{\astr\in\sts}\sum_{t=\swp{i}}^{\swp{i+1}-1}\loss{\astr}{\rft{t}}
\ee
where ties are broken arbitrarily. For all $i\in[\nsws{\cpols}]$ and $t\in[\swp{i},\swp{i+1}-1]\cap\mathbb{N}$ we define $\lst{t}:=\sss{i}$. Clearly we have $\nsws{\lsts}\leq\nsws{\cpols}$.

For all $i\in[\nsws{\cpols}]$ and $\astr\in\sts$ define:
\be
\pra{i}{\astr}:=\prod_{\anod\in\lns}\ssp{i}(\astr(\anod))\,.
\ee
Note that for all $i\in[\nsws{\cpols}]$ we have:
\be
\sum_{\astr\in\sts}\pra{i}{\astr}=1
\ee
and for all $\arf\in\rfs$ we have:
\be
\loss{\ssp{i}}{\arf}=\sum_{\astr\in\sts}\pra{i}{\astr}\loss{\astr}{\arf}
\ee
so that:
\begin{align*}
\sum_{t=\swp{i}}^{\swp{i+1}-1}\loss{\cpolt{t}}{\rft{t}}&=\sum_{t=\swp{i}}^{\swp{i+1}-1}\loss{\ssp{i}}{\rft{t}}\\
&=\sum_{\astr\in\sts}\pra{i}{\astr}\sum_{t=\swp{i}}^{\swp{i+1}-1}\loss{\astr}{\rft{t}}\\
&\geq\min_{\astr\in\sts}\sum_{t=\swp{i}}^{\swp{i+1}-1}\loss{\astr}{\rft{t}}\\
&=\sum_{t=\swp{i}}^{\swp{i+1}-1}\loss{\sss{i}}{\rft{t}}\\
&=\sum_{t=\swp{i}}^{\swp{i+1}-1}\loss{\lst{t}}{\rft{t}}\,.
\end{align*}
Summing this inequality over $i\in[\nsws{\cpols}]$ gives us the result.
\end{proof}

Due to Lemma \ref{poltostlem} we will, from here on, fix a sequence of strategies $\lsts\in\sts^T$. For each $t\in[T]$ we define $\rac{t}$ to be the set of nodes in $\acs$ that can possibly be encountered if the learner were to play the game against any environment according to strategy $\lst{t}$. Formally, $\rac{t}$ is defined as the minimal subset of $\acs$ in which:
\begin{itemize}
\item $\lst{t}(\rot)\in\rac{t}$\,.
\item For all $\ac\in\rac{t}$ and for all $\anod\in\ch{\ac}\cap\lns$ we have $\lst{t}(\anod)\in\rac{t}$\,.
\end{itemize}
We also define:
\be
\rac{0}:=\emptyset\,.
\ee
Given $t\in[T]$ and $\ac\in\acs$ we define $\mgp{t}{\ac}$ as follows:
\begin{itemize}
\item If $\ac\in\ch{\vet{t}{i}}$ for some $i\in[\plt{t}]$ then:
\be
\mgp{t}{\ac}:=\frac{1}{\pst{t}{i-1}^{\ime(\act{t}{i})}}\left(\indi{\ac=\act{t}{i}}\pst{t}{i}^{\ime(\act{t}{i})}+\indi{\ac\neq\act{t}{i}}\right)\,.
\ee
\item If there does not exist $i\in[\plt{t}]$ with $\ac\in\ch{\vet{t}{i}}$ then:
\be
\mgp{t}{\ac}:=1\,.
\ee
\end{itemize}
Given $t\in[T]$ and $\ac\in\acs$ define:
\be
\mgt{t}{\ac}:=\indi{\ac\in\rac{t}}\mgp{t}{\ac}+\indi{\ac\notin\rac{t}}
\ee
and:
\be
\sht{t}{\ac}:=\indi{\ac\in\rac{t}}(1-\shp)+\indi{\ac\notin\rac{t}}
\ee
and:
\be
\wgt{t}{\ac}:=\indi{\ac\notin\rac{t-1}\wedge\ac\in\rac{t}}\frac{\shp}{\nma}+\indi{\ac\in\rac{t-1}\vee\ac\notin\rac{t}}\,.
\ee

\begin{lemma}\label{bboolem}
For all $t\in[T]$ and $i\in[\plt{t}]\cup\{0\}$ we have:
\be
\pst{t}{i}\in(0,1]\,.
\ee
\end{lemma}

\begin{proof}
We prove by backwards induction on $i$. Since $\lr>0$, $\lost{t}\geq0$ and $\polt{t}(\act{t}{k})>0$ for all $k\in[\plt{t}]$\,, we have $\pst{t}{\plt{t}}\in (0,1]$. Now suppose that we have some $j\in[\plt{t}]$ such that $\pst{t}{j}\in(0,1]$. We will now show that $\pst{t}{j-1}\in(0,1]$ which will complete the inductive proof.

Since $\pst{t}{j}\in(0,1]$ and $\ime(\act{t}{j})>0$ we have:
\be
1-\pst{t}{j}^{\ime(\act{t}{j})}\in[0,1)
\ee
so since $\polt{t}(\act{t}{j})\in(0,1]$ we have:
\be
1-(1-\pst{t}{j}^{\ime(\act{t}{j})})\polt{t}(\act{t}{j})\in(0,1]\,.
\ee
Since $1/\ime(\act{t}{j})>0$ we then have that $\pst{t}{j-1}\in(0,1]$ as required. This completes the inductive proof.
\end{proof}

Note that by Lemma \ref{bboolem} we have, for all $t\in[T]$ and $\ac\in\acs$, that $\mgp{t}{\ac}>0$ and hence that $\mgt{t}{\ac}>0$ so that $\ln(\mgt{t}{\ac})$ exists as a real number. We will use this ability to take the logarithm throughout this analysis.

\begin{lemma}\label{lesstolem}
For all $\ac\in\acs$ we have:
\be
\prod_{t\in[T]}\sht{t}{\ac}\wgt{t}{\ac}\mgt{t}{\ac}\leq 1\,.
\ee
\end{lemma}

\begin{proof}
We take the inductive hypothesis that for all $s\in[T]\cup\{0\}$ we have:
\be
\prod_{t\in[s]}\sht{t}{\ac}\wgt{t}{\ac}\mgt{t}{\ac}\leq\indi{\ac\in\rac{s}}\polt{s+1}(\ac)+\indi{\ac\notin\rac{s}}
\ee
and prove by induction over $s$. The inductive hypothesis clearly holds for $s=0$ as $\ac\notin\rac{0}$. Now suppose we have some $q\in[T-1]\cup\{0\}$ such that the inductive hypothesis holds for $s=q$. We now show that it holds for $s=q+1$ which will complete the proof. We have the following cases:
\begin{itemize}
\item The first case is that $\ac\notin\rac{q+1}$. Here we have $\sht{q+1}{\ac}=1$, $\wgt{q+1}{\ac}=1$ and $\mgt{q+1}{\ac}=1$ and hence, by the inductive hypothesis we have:
\begin{align*}
\prod_{t\in[q+1]}\sht{t}{\ac}\wgt{t}{\ac}\mgt{t}{\ac}&=\prod_{t\in[q]}\sht{t}{\ac}\wgt{t}{\ac}\mgt{t}{\ac}\\
&\leq\indi{\ac\in\rac{q}}\polt{q+1}(\ac)+\indi{\ac\notin\rac{q}}\\
&\leq 1
\end{align*}
as required.
\item The second case is that $\ac\notin\rac{q}$ and $\ac\in\rac{q+1}$. Here we have $\sht{q+1}{\ac}=1-\shp$\,, $\wgt{q+1}{\ac}=\shp/\nma$ and $\mgt{q+1}{\ac}=\mgp{q+1}{\ac}$. Since, directly from the algorithm, we have that $\polt{q+1}(\ac)\geq\shp/\nma$ we then have $\wgt{q+1}{\ac}\leq\polt{q+1}(\ac)$. Directly from the algorithm we have: 
\be
(1-\shp)\mgp{q+1}{\ac}\polt{q+1}(\ac)\leq\polt{q+2}(\ac)\,.
\ee
Putting together gives us:
\be
\sht{q+1}{\ac}\wgt{q+1}{\ac}\mgt{q+1}{\ac}\leq\polt{q+2}(\ac)
\ee
So by the inductive hypothesis we have the result.
\item The third case is that $\ac\in\rac{q}$ and $\ac\in\rac{q+1}$. Here we have that $\sht{q+1}{\ac}=1-\shp$\,, $\wgt{q+1}{\ac}=1$ and $\mgt{q+1}{\ac}=\mgp{q+1}{\ac}$. Directly from the algorithm we have: 
\be
(1-\shp)\mgp{q+1}{\ac}\polt{q+1}(\ac)\leq\polt{q+2}(\ac)\,.
\ee
Putting together gives us:
\be
\sht{q+1}{\ac}\wgt{q+1}{\ac}\mgt{q+1}{\ac}\polt{q+1}(\ac)\leq\polt{q+2}(\ac)
\ee
So by the inductive hypothesis we have the result. 
\end{itemize}
We have now proved that the inductive hypothesis holds for $s=q+1$ and hence that it holds for all $s\in[T]\cup\{0\}$. In particular it holds for $s=T$ which gives us the result.
\end{proof}

\begin{lemma}\label{ooimestolem}
For all $t\in[T]$ we have:
\be
\sum_{\ac\in\rac{t}}\frac{1}{\ime(\ac)}=1\,.
\ee
\end{lemma}

\nc{\rln}[1]{\mathcal{N}^\dag_{#1}}
\nc{\acdes}[2]{\mathcal{A}'_{#1}(#2)}

\begin{proof}
Let $\rln{t}:=\{\pnt{\ac}\,|\,\ac\in\rac{t}\}$. Given $\anod\in\rln{t}$ let $\acdes{t}{\anod}$ be the set of nodes in $\rac{t}$ that are descendants of $\anod$. We take the inductive hypothesis that for each $\anod\in\rln{t}$ we have:
\be
\sum_{\ac\in\acdes{t}{\anod}}\frac{1}{\ime(\ac)}=\frac{1}{\ime(\anod)}
\ee
and prove by induction up the tree. Hence, we may assume that for all $\anod'\in\ch{\lst{t}(\anod)}\cap\lns$ we have:
\be
\sum_{\ac\in\acdes{t}{\anod'}}\frac{1}{\ime(\ac)}=\frac{1}{\ime(\anod')}
\ee
which gives us:
\begin{align*}
\sum_{\ac\in\acdes{t}{\anod}}\frac{1}{\ime(\ac)}&=\frac{1}{\ime(\lst{t}(\anod))}+\sum_{\anod'\in\ch{\lst{t}(\anod)}\cap\lns}\sum_{\ac\in\acdes{t}{\anod'}}\frac{1}{\ime(\ac)}\\
&=\frac{1}{\ime(\lst{t}(\anod))}+\sum_{\anod'\in\ch{\lst{t}(\anod)}\cap\lns}\frac{1}{\ime(\anod')}\\
&=\frac{1}{\ime(\lst{t}(\anod))}+\sum_{\anod'\in\ch{\lst{t}(\anod)}\cap\lns}\frac{\msn(\anod')}{\ime(\lst{t}(\anod))}\\
&=\frac{\msn(\lst{t}(\anod))}{\ime(\lst{t}(\anod))}\\
&=\frac{1}{\ime(\anod)}
\end{align*}
as required.

We have now proved that the inductive hypothesis holds for all $\anod\in\rln{t}$ and hence that it holds for $\anod=\rot$ which gives us the result.
\end{proof}

\begin{lemma}\label{combolem1}
We have:
\be
\sum_{\ac\in\acs}\frac{1}{\ime(\ac)}\sum_{t\in[T]}\ln(\sht{t}{\ac})=T\ln(1-\shp)
\ee
\end{lemma}

\begin{proof}
By Lemma \ref{ooimestolem}, we have:
\begin{align*}
\sum_{\ac\in\acs}\frac{1}{\ime(\ac)}\sum_{t\in[T]}\ln(\sht{t}{\ac})&=\sum_{\ac\in\acs}\frac{1}{\ime(\ac)}\sum_{t\in[T]}\indi{\ac\in\rac{t}}\ln(1-\shp)\\
&=\ln(1-\shp)\sum_{t\in[T]}\sum_{\ac\in\acs}\indi{\ac\in\rac{t}}\frac{1}{\ime(\ac)}\\
&=\ln(1-\shp)\sum_{t\in[T]}\sum_{\ac\in\rac{t}}\frac{1}{\ime(\ac)}\\
&=\ln(1-\shp)\sum_{t\in[T]}1\\
&=\ln(1-\shp)T
\end{align*}
as required.
\end{proof}

\nc{\aset}{\mathcal{X}}

\begin{lemma}\label{combolem2}
We have:
\be
\sum_{\ac\in\acs}\frac{1}{\ime(\ac)}\sum_{t\in[T]}\ln(\wgt{t}{\ac})\geq\nsws{\lsts}\ln\left(\frac{\shp}{\nma}\right)\,.
\ee
\end{lemma}

\begin{proof}
Let $\aset$ be the set of all $t\in[T]$ such that either $t=1$ or $\lst{t-1}\neq\lst{t}$. Note that $|\aset|=\nsws{\lsts}$.

For all $t\in[T]\setminus\aset$ we have $\rac{t-1}=\rac{t}$ so for all $\ac\in\acs$ we have $\wgt{t}{\ac}=1$. Also, for all $t\in[T]$ and $\ac\in\acs\setminus\rac{t}$ we have $\wgt{t}{\ac}=1$. Hence, by Lemma \ref{ooimestolem}, we have:
\begin{align*}
\sum_{\ac\in\acs}\frac{1}{\ime(\ac)}\sum_{t\in[T]}\ln(\wgt{t}{\ac})&=\sum_{\ac\in\acs}\frac{1}{\ime(\ac)}\sum_{t\in\aset}\ln(\wgt{t}{\ac})\\
&=\sum_{t\in\aset}\sum_{\ac\in\acs}\frac{1}{\ime(\ac)}\ln(\wgt{t}{\ac})\\
&=\sum_{t\in\aset}\sum_{\ac\in\rac{t}}\frac{1}{\ime(\ac)}\ln(\wgt{t}{\ac})\\
&\geq\sum_{t\in\aset}\sum_{\ac\in\rac{t}}\frac{1}{\ime(\ac)}\ln(\shp/\nma)\\
&=\ln(\shp/\nma)\sum_{t\in\aset}\sum_{\ac\in\rac{t}}\frac{1}{\ime(\ac)}\\
&=\ln(\shp/\nma)\sum_{t\in\aset}1\\
&=\ln(\shp/\nma)\nsws{\lsts}
\end{align*}
as required.
\end{proof}

\nc{\ple}[1]{e_{#1}}
\nc{\vey}[2]{y_{#1,#2}}
\nc{\vc}[2]{c_{#1,#2}}
\nc{\cprob}[2]{\mathbb{P}[#1\,|\,#2]}

\begin{lemma}\label{cancellem}
For all $t\in[T]$ we have:
\be
\expt{\sum_{\ac\in\acs}\frac{\ln(\mgt{t}{\ac})}{\ime(\ac)}}=-\lr\loss{\lst{t}}{\rft{t}}-\expt{\pst{t}{0}}\,.
\ee
\end{lemma}

\begin{proof}

Let $\ple{t}$ be the number of nodes in $\lns$ that would be encountered if the learner were to play the game against $\rft{t}$ according to strategy $\lst{t}$. For all $i\in[\ple{t}]$ let $\vey{t}{i}\in\lns$ be the $i$-th node in $\lns$ that would be encountered if the learner were to play the game against $\rft{t}$ according to strategy $\lst{t}$. For all $i\in[\ple{t}]$ let $\vc{t}{i}:=\lst{t}(\vey{t}{i})$

Take any $\ac\in\rac{t}$ with $\pnt{\ac}=\vet{t}{i}$ for some $i\in[\plt{t}]$. Assume, for contradiction, that $\pnt{\ac}\neq\vey{t}{i}$. Then let $j$ be the maximal element of $[i-1]$ such that $\vet{t}{j}=\vey{t}{j}$, which exists as $\vet{t}{1}=\rot=\vey{t}{1}$. As $\pnt{\ac}=\vet{t}{i}$ we must have that $\ac$ is a descendant of $\act{t}{j}$. But also, as $\ac\in\rac{t}$ and is a descendant of $\vet{t}{j}$ we must have that $\ac$ is a descendant of $\lst{t}(\vet{t}{j})$ so that $\act{t}{j}=\lst{t}(\vet{t}{j})=\vc{t}{j}$. But $\rft{t}(\act{t}{j})=\vet{t}{j+1}$ and $\rft{t}(\vc{t}{j})=\vey{t}{j+1}$ so that $\vet{t}{j+1}=\vey{t}{j+1}$ which is a contradiction. We have now shown that $\pnt{\ac}=\vey{t}{i}$ so since $\ac\in\rac{t}$ we must have $\ac=\lst{t}(\vey{t}{i})=\vc{t}{i}$.

We have now shown that for all $\ac\in\rac{t}$ such that there does not exist $i\in[\ple{t}]$ with $\ac=\vc{t}{i}$, we have $\pnt{\ac}\neq\vet{t}{j}$ for all $j\in[\plt{t}]$ and hence that $\mgt{t}{\ac}=\mgp{t}{\ac}=1$. Also, for all $\ac\in\acs\setminus\rac{t}$ we have $\mgt{t}{\ac}=1$. We have hence shown that:
\begin{equation}\label{cancellemeq1}
\sum_{\ac\in\acs}\frac{1}{\ime(\ac)}\ln(\mgt{t}{\ac})=\sum_{i\in[\ple{t}]}\frac{1}{\ime(\vc{t}{i})}\ln(\mgp{t}{\vc{t}{i}})\,.
\end{equation}
For all $i\in[\ple{t}]$ we have:
\be
\frac{1}{\ime(\vc{t}{i})}\ln(\mgp{t}{\vc{t}{i}})
=\indi{\vey{t}{i}=\vet{t}{i}}\left(\indi{\vc{t}{i}=\act{t}{i}}\ln(\pst{t}{i})-\ln(\pst{t}{i-1})\right)
\ee
so since $\vc{t}{i}=\act{t}{i}$ implies $\vey{t}{i}=\vet{t}{i}$ we have:
\begin{equation}\label{cancellemeq2}
\frac{1}{\ime(\vc{t}{i})}\ln(\mgp{t}{\vc{t}{i}})
=\indi{\vc{t}{i}=\act{t}{i}}\ln(\pst{t}{i})-\indi{\vey{t}{i}=\vet{t}{i}}\ln(\pst{t}{i-1})\,.
\end{equation}
For all $i\in[\ple{t}-1]$ we have that $\vc{t}{i}=\act{t}{i}$ implies $\vey{t}{i+1}=\rft{t}(\vc{t}{i})=\rft{t}(\act{t}{i})=\vet{t}{i+1}$. Clearly also we have that $\vey{t}{i+1}=\vet{t}{i+1}$ implies $\vc{t}{i}=\act{t}{i}$ so that:
\be
\indi{\vc{t}{i}=\act{t}{i}}=\indi{\vey{t}{i+1}=\vet{t}{i+1}}\,.
\ee
Substituting into Equation \ref{cancellemeq2} gives us, for all $i\in[\ple{t}-1]$, that:
\begin{equation}\label{cancellemeq3}
\frac{1}{\ime(\vc{t}{i})}\ln(\mgp{t}{\vc{t}{i}})
=\indi{\vey{t}{i+1}=\vet{t}{i+1}}\ln(\pst{t}{i})-\indi{\vey{t}{i}=\vet{t}{i}}\ln(\pst{t}{i-1})\,.
\end{equation}
Utilising equations \eqref{cancellemeq2} and \eqref{cancellemeq3} in a telescopic sum gives us:
\begin{equation}\label{cancellemeq4}
\sum_{i\in[\ple{t}]}\frac{1}{\ime(\vc{t}{i})}\ln(\mgp{t}{\vc{t}{i}})=\indi{\vc{t}{\ple{t}}=\act{t}{\ple{t}}}\ln(\pst{t}{\ple{t}})-\indi{\vey{t}{1}=\vet{t}{1}}\ln(\pst{t}{0})\,.
\end{equation}
Note that if $\vc{t}{\ple{t}}=\act{t}{\ple{t}}$ then $\rft{t}(\act{t}{\ple{t}})=\rft{t}(\vc{t}{\ple{t}})\in\tns$ so that $\ple{t}=\plt{t}$. Also note that if $\vc{t}{\ple{t}}=\act{t}{\ple{t}}$ then $\vc{t}{i}=\act{t}{i}$ for all $i\in[\ple{t}]$. Also note that if $\vc{t}{\ple{t}}=\act{t}{\ple{t}}$ then, since (by above) $\ple{t}=\plt{t}$, we have:
\be
\lost{t}=\losf(\rft{t}(\act{t}{\ple{t}}))=\losf(\rft{t}(\vc{t}{\ple{t}}))=\loss{\lst{t}}{\rft{t}}\,.
\ee
Hence, we have that:
\begin{align*}
\indi{\vc{t}{\ple{t}}=\act{t}{\ple{t}}}\ln(\pst{t}{\ple{t}})&=\indi{\vc{t}{\ple{t}}=\act{t}{\ple{t}}}\ln(\pst{t}{\plt{t}})\\
&=\frac{-\indi{\vc{t}{\ple{t}}=\act{t}{\ple{t}}}\lr\lost{t}}{\prod_{i\in[\plt{t}]}\polt{t}(\act{t}{i})}\\
&=\frac{-\indi{\vc{t}{\ple{t}}=\act{t}{\ple{t}}}\lr\lost{t}}{\prod_{i\in[\ple{t}]}\polt{t}(\act{t}{i})}\\
&=\frac{-\indi{\vc{t}{\ple{t}}=\act{t}{\ple{t}}}\lr\lost{t}}{\prod_{i\in[\ple{t}]}\polt{t}(\vc{t}{i})}\\
&=\frac{-\indi{\vc{t}{\ple{t}}=\act{t}{\ple{t}}}\lr\loss{\lst{t}}{\rft{t}}}{\cprob{\vc{t}{\ple{t}}=\act{t}{\ple{t}}}{\polt{t}}}\,.
\end{align*}
Substituting into Equation \eqref{cancellemeq4} and noting that $\vey{t}{1}=\rot=\vet{t}{1}$ gives us:
\be
\sum_{i\in[\ple{t}]}\frac{1}{\ime(\vc{t}{i})}\ln(\mgp{t}{\vc{t}{i}})= \frac{-\indi{\vc{t}{\ple{t}}=\act{t}{\ple{t}}}\lr\loss{\lst{t}}{\rft{t}}}{\cprob{\vc{t}{\ple{t}}=\act{t}{\ple{t}}}{\polt{t}}}-\ln(\pst{t}{0})
\ee
and hence:
\be
\mathbb{E}\left[\sum_{i\in[\ple{t}]}\frac{1}{\ime(\vc{t}{i})}\ln(\mgp{t}{\vc{t}{i}})\,\Bigg|\,\polt{t}\right]=-\lr\loss{\lst{t}}{\rft{t}}-\mathbb{E}[\ln(\pst{t}{0})\,|\,\polt{t}]
\ee
using Equation \ref{cancellemeq1}, the result follows.
\end{proof}

\begin{lemma}\label{pstzlem}
For all $t\in[T]$ we have:
\be
\ln(\pst{t}{0})\leq \frac{\lr^2}{2}\sum_{i\in[\plt{t}]}\frac{\ime(\act{t}{i})}{\prod_{j\in[i]}\polt{t}(\act{t}{j})}-\lr\lost{t}\,.
\ee
\end{lemma}

\nc{\arn}{z}
\nc{\fpa}[2]{f'_{#1,#2}}
\nc{\spa}[2]{f_{#1,#2}}

\begin{proof}
We take the inductive hypothesis that for all $k\in[\plt{t}]\cup\{0\}$ we have:
\be
\ln(\pst{t}{k})\leq\frac{\lr^2\lost{t}^2}{2}\sum_{i=k+1}^{\plt{t}}\frac{\ime(\act{t}{i})}{\left(\prod_{j\in[k]}\polt{t}(\act{t}{i})\right)\left(\prod_{j\in[i]}\polt{t}(\act{t}{i})\right)}-\frac{\lr\lost{t}}{\prod_{j\in[k]}\polt{t}(\act{t}{i})}
\ee
and prove by backwards induction on $k$. The inductive hypothesis clearly holds (with equality) for $k=\plt{t}$. Now assume that we have $h\in[\plt{t}]$ such that the inductive hypothesis holds for $k=h$. We will now show that it holds for $k=h-1$ which will complete the proof of the inductive hypothesis.

Define:
\be
\fpa{t}{h}:=\frac{\lr^2\lost{t}^2}{2}\sum_{i=h+1}^{\plt{t}}\frac{\ime(\act{t}{i})}{\left(\prod_{j\in[h]}\polt{t}(\act{t}{i})\right)\left(\prod_{j\in[i]}\polt{t}(\act{t}{i})\right)}
\ee
and:
\be
\spa{t}{h}:=\frac{\lr\lost{t}}{\prod_{j\in[h]}\polt{t}(\act{t}{i})}
\ee
noting that both these terms are positive and:
\be
\ln(\pst{t}{h})\leq\fpa{t}{h}-\spa{t}{h}\,.
\ee
We have the following two cases:
\begin{itemize}
\item In the first case we have $\fpa{t}{h}\geq\spa{t}{h}$. By Lemma \ref{bboolem} we have, since $\ime(\act{t}{h})>0$, that:
\begin{align*}
\pst{t}{h}^{\ime(\act{t}{h})}&\leq 1\\
&\leq 1 + \ime(\act{t}{h})(\fpa{t}{h} - \spa{t}{h})\\
&\leq 1 + \ime(\act{t}{h})(\fpa{t}{h} - \spa{t}{h}) + \frac{1}{2}\ime(\act{t}{h})^2\spa{t}{h}^2\,.
\end{align*}
\item In the second case we have $\fpa{t}{h}<\spa{t}{h}$. In this case we have, since $\ime(\act{t}{h})>0$, that:
\begin{align*}
\pst{t}{h}^{\ime(\act{t}{h})}&=\exp(\ime(\act{t}{h})\ln(\pst{t}{h}))\\
&\leq\exp(\ime(\act{t}{h})(\fpa{t}{h} - \spa{t}{h}))
\end{align*}
so since $\exp(z)\leq 1 +z+z^2/2$ for all $z\leq 0$ we have:
\begin{align*}
\pst{t}{h}^{\ime(\act{t}{h})}&\leq 1+\ime(\act{t}{h})(\fpa{t}{h} - \spa{t}{h})+\frac{1}{2}\ime(\act{t}{h})^2(\fpa{t}{h} - \spa{t}{h})^2\\
&\leq1+\ime(\act{t}{h})(\fpa{t}{h} - \spa{t}{h}) +\frac{1}{2}\ime(\act{t}{h})^2\spa{t}{h}^2\,.
\end{align*}
\end{itemize}
So in either case we have:
\be
\pst{t}{h}^{\ime(\act{t}{h})}\leq1+\ime(\act{t}{h})(\fpa{t}{h} - \spa{t}{h}) +\frac{1}{2}\ime(\act{t}{h})^2\spa{t}{h}^2
\ee
so that, since $\ime(\act{t}{h})>0$, we have:
\begin{align*}
\ln(\pst{t}{h-1})&=\frac{1}{\ime(\act{t}{h})}\ln\left(1-\left(1-\pst{t}{h}^{\ime(\act{t}{h})}\right)\polt{t}(\act{t}{h})\right)\\
&\leq\frac{1}{\ime(\act{t}{h})}\ln\left(1+\ime(\act{t}{h})(\fpa{t}{h} - \spa{t}{h})\polt{t}(\act{t}{h}) +\frac{1}{2}\ime(\act{t}{h})^2\spa{t}{h}^2\polt{t}(\act{t}{h})\right)
\end{align*}
and hence, since $\ln(z)\leq z-1$ for all $z>0$, we have:
\be
\ln(\pst{t}{h-1})\leq\fpa{t}{h}\polt{t}(\act{t}{h})-\spa{t}{h}\polt{t}(\act{t}{h})+\frac{1}{2}\ime(\act{t}{h})\spa{t}{h}^2\polt{t}(\act{t}{h})
\ee
which proves the inductive hypothesis holds for $k=h-1$.

We have hence proved that the inductive hypothesis holds always. In particular it holds for $k=0$ which, noting that $\lost{t}^2\leq1$, gives us the result.

\end{proof}

\begin{lemma}\label{mrialem}
We have:
\be
\msn(\rot)=\tnma\,.
\ee
\end{lemma}

\begin{proof}
By a simple induction up the tree we have that for any $\anod\in\lns\cup\acs$\,, $\msn(\anod)$ is the number of nodes in $\acs$ that are descendants of $\anod$. The result follows immediately.
\end{proof}

\nc{\lay}[2]{\mathcal{D}_{#1,#2}}
\nc{\racr}[1]{\mathcal{A}^*_{#1}}

For all $t\in[T]$ and $i\in[\depth]$ we define $\lay{t}{i}$ inductively as follows:
\begin{itemize}
\item $\lay{t}{1}:=\ch{\rot}$
\item For all $j\in[\depth-1]$ we have 
\be
\lay{t}{j+1}:=\bigcup_{\ac\in\lay{t}{j}}\ch{\rft{t}(\ac)}\,.
\ee
\end{itemize}
We define: 
\be
\racr{t}:=\bigcup_{i\in[\depth]}\lay{t}{i}
\ee
which is the set of nodes in $\acs$ that can possibly be encountered if the learner plays the game against environment $\rft{t}$.

\begin{lemma}\label{soimelem}
For all $t\in[T]$ we have:
\be
\sum_{\ac\in\racr{t}}\ime(\ac)\leq\depth\tnma\,.
\ee
\end{lemma}

\nc{\srac}[2]{\mathcal{A}_{#1}^\circ(#2)}

\begin{proof}
Given any $\ac\in\racr{t}$ define $\srac{t}{\ac}$ to be the set of nodes in $\racr{t}$ that are descendants of $\ac$.

We take the inductive hypothesis that for all $i\in[\depth]$ and all $\ac\in\lay{t}{i}$ we have:
\be
\sum_{\ac'\in\srac{t}{\ac}}\ime(\ac')\leq(\depth-i+1)\ime(\ac)
\ee
and prove by backward induction on $i$. We clearly have the inductive hypothesis for $i=\depth$ as for any $\ac\in\lay{t}{\depth}$ we have $\srac{t}{\ac}=\{\ac\}$ so:
\be
\sum_{\ac'\in\srac{t}{\ac}}\ime(\ac')=\ime(\ac)
\ee
as required. Now suppose that we have some $j\in[\depth]\setminus\{1\}$ such that the inductive hypothesis holds for $k=j$. We now show that the inductive hypothesis holds for $k=j-1$ which will prove that the inductive hypothesis holds always.

Take any $\ac\in\lay{t}{j-1}$. We have the following two cases:
\begin{itemize}
\item We first consider the case that $\rft{t}(\ac)\in\tns$. In this case we have that $\srac{t}{\ac}=\{\ac\}$ so:
\begin{align*}
\sum_{\ac'\in\srac{t}{\ac}}\ime(\ac')&=\ime(\ac)\\
&<(\depth - (j-1) +1)\ime(\ac)
\end{align*}
as required.
\item We next consider the case that $\rft{t}(\ac)\in\lns$. Since:
\be
\srac{t}{\ac}=\{\ac\}\cup\bigcup_{\ac'\in\ch{\rft{t}(\ac)}}\srac{t}{\ac'}
\ee
and $\ch{\rft{t}(\ac)}\subseteq\lay{t}{j}$, we have:
\begin{align*}
\sum_{\ac'\in\srac{t}{\ac}}\ime(\ac')&=\ime(\ac)+\sum_{\ac'\in\ch{\rft{t}(\ac)}}\sum_{\ac''\in\srac{t}{\ac'}}\ime(\ac'')\\
&\leq\ime(\ac)+(\depth-j+1)\sum_{\ac'\in\ch{\rft{t}(\ac)}}\ime(\ac')\\
&=\ime(\ac)+(\depth-j+1)\sum_{\ac'\in\ch{\rft{t}(\ac)}}\msn(\ac')\ime(\rft{t}(\ac))\\
&=\ime(\ac)+(\depth-j+1)\ime(\rft{t}(\ac))\sum_{\ac'\in\ch{\rft{t}(\ac)}}\msn(\ac')\\
&=\ime(\ac)+(\depth-j+1)\ime(\rft{t}(\ac))\msn(\rft{t}(\ac))\\
&=\ime(\ac)+(\depth-j+1)\ime(\ac)\\
&=(\depth-(j-1)+1)\ime(\ac)
\end{align*}
as required.
\end{itemize}
We have hence shown that the inductive hypothesis holds for $i=j-1$ and hence that it holds always. In particular it holds for $i=1$. So since:
\be
\racr{t}=\bigcup_{\ac\in\ch{\rot}}\srac{t}{\ac}
\ee
and $\ch{\rot}=\lay{t}{1}$, we have, by Lemma \ref{mrialem}, that:
\begin{align*}
\sum_{\ac\in\racr{t}}\ime(\ac)&=\sum_{\ac\in\ch{\rot}}\sum_{\ac'\in\srac{t}{\ac}}\ime(\ac')\\
&=\sum_{\ac\in\ch{\rot}}(\depth-1+1)\ime(\ac)\\
&=\depth\sum_{\ac\in\ch{\rot}}\ime(\ac)\\
&=\depth\sum_{\ac\in\ch{\rot}}\msn(\ac)\ime(\rot)\\
&=\depth\sum_{\ac\in\ch{\rot}}\msn(\ac)\\
&=\depth\msn(\rot)\\
&=\depth\tnma
\end{align*}
as required.

\end{proof}

\nc{\anc}[2]{p^\dag_{#1}(#2)}

\begin{lemma}\label{elpstblem}
For all $t\in[T]$ we have:
\be
\expt{\ln(\pst{t}{0})}\leq\frac{\lr^2}{2}\depth\tnma-\lr\expt{\lost{t}}\,.
\ee
\end{lemma}

\begin{proof}
For all $i\in[\depth]$, $j\in[i]$ and $\ac\in\lay{t}{i}$ let $\anc{j}{\ac}$ be the unique ancestor of $\ac$ that is contained in $\lay{t}{j}$.

Take any $i\in[\depth]$. First note that the only nodes in $\acs$ that can possibly be equal to $\act{t}{i}$ are the nodes in $\lay{t}{i}$. Note also that for any $\ac\in\lay{t}{i}$, if $\ac=\act{t}{i}$ then for all $j\in[i]$ we have $\act{t}{j}=\anc{j}{\ac}$. Hence, we have:
\begin{align*}
\mathbb{E}\left[\frac{\indi{i\leq\plt{t}}\ime(\act{t}{i})}{\prod_{j\in[i]}\polt{t}(\act{t}{i})}\,\Bigg|\,\polt{t}\right]&=\sum_{\ac\in\lay{t}{i}}\frac{\cprob{\ac=\act{t}{i}}{\polt{t}}\ime(\ac)}{\prod_{j\in[i]}\polt{t}(\anc{j}{\ac})}\\
&=\sum_{\ac\in\lay{t}{i}}\frac{\cprob{\ac=\act{t}{i}}{\polt{t}}\ime(\ac)}{\cprob{\ac=\act{t}{i}}{\polt{t}}}\\
&=\sum_{\ac\in\lay{t}{i}}\ime(\ac)\,.
\end{align*}
By lemmas \ref{soimelem} and \ref{pstzlem} we then have:
\begin{align*}
\mathbb{E}[\ln(\pst{t}{0})\,|\,\polt{t}]&\leq\frac{\lr^2}{2}\sum_{i\in[\depth]}\mathbb{E}\left[\frac{\indi{i\leq\plt{t}}\ime(\act{t}{i})}{\prod_{j\in[i]}\polt{t}(\act{t}{i})}\,\Bigg|\,\polt{t}\right]-\lr\mathbb{E}[\lost{t}\,|\,\polt{t}]\\
&=\frac{\lr^2}{2}\sum_{i\in[\depth]}\sum_{\ac\in\lay{t}{i}}\ime(\ac)-\lr\mathbb{E}[\lost{t}\,|\,\polt{t}]\\
&=\frac{\lr^2}{2}\sum_{\ac\in\racr{t}}\ime(\ac)-\lr\mathbb{E}[\lost{t}\,|\,\polt{t}]\\
&=\frac{\lr^2}{2}\depth\tnma - \lr\mathbb{E}[\lost{t}\,|\,\polt{t}]\\
\end{align*}
which implies the result.

\end{proof}

\begin{lemma}\label{finlem}
We have:
\be
\mathbb{E}\left[\sum_{t\in[T]}\lost{t}\right]-\sum_{t\in[T]}\loss{\lst{t}}{\rft{t}}\in\mathcal{O}\left(\left(\frac{1}{\lrh}+\lrh\nsws{\lsts}\right)\sqrt{\depth \tnma T\ln(\lrh^2T)}\right)\,.
\ee
\end{lemma}

\begin{proof}
From Lemma \ref{lesstolem} we have, for all $\ac\in\acs$, that:
\begin{align*}
\sum_{t\in[T]}\ln(\sht{t}{\ac})+\sum_{t\in[T]}\ln(\wgt{t}{\ac})+\sum_{t\in[T]}\ln(\mgt{t}{\ac})&=\ln\left(\prod_{t\in[T]}\sht{t}{\ac}\wgt{t}{\ac}\mgt{t}{\ac}\right)\\
&\leq 0
\end{align*}
so that:
\begin{align*}
&\sum_{\ac\in\acs}\frac{1}{\ime(\ac)}\sum_{t\in[T]}\ln(\sht{t}{\ac})+\sum_{\ac\in\acs}\frac{1}{\ime(\ac)}\sum_{t\in[T]}\ln(\wgt{t}{\ac})+\sum_{t\in[T]}\sum_{\ac\in\acs}\frac{\ln(\mgt{t}{\ac})}{\ime(\ac)}\\
=&\sum_{\ac\in\acs}\frac{1}{\ime(\ac)}\left(\sum_{t\in[T]}\ln(\sht{t}{\ac})+\sum_{t\in[T]}\ln(\wgt{t}{\ac})+\sum_{t\in[T]}\ln(\mgt{t}{\ac})\right)\\
\leq&0
\end{align*}
and hence, by lemmas \ref{combolem1} and \ref{combolem2}, we have:
\be
T\ln(1-\shp)+\nsws{\lsts}\ln\left(\frac{\shp}{\nma}\right)+\sum_{t\in[T]}\sum_{\ac\in\acs}\frac{\ln(\mgt{t}{\ac})}{\ime(\ac)}\leq0
\ee
so that, by lemmas \ref{cancellem} and \ref{elpstblem} we have:
\begin{align*}
T\ln(1-\shp)+\nsws{\lsts}\ln\left(\frac{\shp}{\nma}\right)&\leq-\mathbb{E}\left[\sum_{t\in[T]}\sum_{\ac\in\acs}\frac{\ln(\mgt{t}{\ac})}{\ime(\ac)}\right]\\
&=-\sum_{t\in[T]}\mathbb{E}\left[\sum_{\ac\in\acs}\frac{\ln(\mgt{t}{\ac})}{\ime(\ac)}\right]\\
&=\sum_{t\in[T]}\left(\lr\loss{\lst{t}}{\rft{t}}+\expt{\pst{t}{0}}\right)\\
&\leq\sum_{t\in[T]}\left(\lr\loss{\lst{t}}{\rft{t}}+\frac{\lr^2}{2}\depth\tnma-\lr\expt{\lost{t}}\right)\\
&=\lr\sum_{t\in[T]}\loss{\lst{t}}{\rft{t}}+\frac{\lr^2}{2}\depth\tnma T - \lr\mathbb{E}\left[\sum_{t\in[T]}\lost{t}\right]\,.
\end{align*}
Rearranging gives us:
\be
\mathbb{E}\left[\sum_{t\in[T]}\lost{t}\right]-\sum_{t\in[T]}\loss{\lst{t}}{\rft{t}}\leq-\frac{1}{\lr}\left(T\ln(1-\shp)+\nsws{\lsts}\ln\left(\frac{\shp}{\nma}\right)\right)+\frac{\lr}{2}\depth\tnma T\,.
\ee
Now note that $\ln(1-\shp)=\ln(1-1/\lrh^2T)\in\mathcal{O}(-1/\lrh^2T)$ and since without loss of generality $1/\shp=\lrh^2T\geq \tnma\geq\nma$ (else the bound is vacuous) we have $\ln(\shp/\nma)\geq2\ln(\shp)\in\mathcal{O}(-\ln(\lrh^2T))$. Hence, we have:
\be
\mathbb{E}\left[\sum_{t\in[T]}\lost{t}\right]-\sum_{t\in[T]}\loss{\lst{t}}{\rft{t}}\in\mathcal{O}\left(\frac{1}{\lr\lrh^2}+\frac{1}{\lr}\nsws{\lsts}\ln(\lrh^2T)+\frac{\lr}{2}\depth\tnma T\right)
\ee
so since, by Lemma \ref{mrialem}, we have:
\be
\lr=\frac{1}{\lrh}\sqrt{\frac{2\ln(\lrh^2T)}{\depth \tnma T}}
\ee
we have the result.

\end{proof}

For any sequence of policies $\cpols\in\pols^T$, lemmas \ref{poltostlem} and \ref{finlem} give us:
\be
\mathbb{E}\left[\sum_{t\in[T]}\lost{t}\right]-\sum_{t\in[T]}\loss{\cpolt{t}}{\rft{t}}\in\mathcal{O}\left(\left(\frac{1}{\lrh}+\lrh\nsws{\cpols}\right)\sqrt{\depth \tnma T\ln(\lrh^2T)}\right)
\ee
which, since the computational complexity of \alg\ is immediate, completes the proof of Theorem \ref{mainth}.

\bibliographystyle{tmlr}
\bibliography{bib}

\end{document}